\documentclass[10pt]{IEEEtran}

\usepackage{amssymb}
\usepackage{amsmath,mathtools,steinmetz}
\usepackage{amsthm}
\usepackage{stmaryrd}
\usepackage{paralist}
\usepackage{tikz,pgfplots,subfigure}
\usepackage{eurosym}
\usepackage{epstopdf,multicol,float}
\usepackage[hide links]{hyperref}
\usepackage{cleveref}
\usepackage{bm}
\usepackage{upgreek}
\usepackage{dsfont}
\usepackage{color}

\makeatletter
\def\ps@IEEEtitlepagestyle{%
  \def\@oddfoot{\mycopyrightnotice}%
  \def\@evenfoot{}%
}
\def\mycopyrightnotice{%
  {\footnotesize This work has been submitted to the IEEE for possible publication. Copyright may be transferred without notice, after which this version may no longer be accessible.\hfill}
    \gdef\mycopyrightnotice{}
  }

\IEEEoverridecommandlockouts

\newtheorem{prop}{Proposition}
\newtheorem{thm}{Theorem}
\newtheorem{cor}{Corollary}
\newtheorem{lem}{Lemma}
\theoremstyle{remark}
\newtheorem{rem}{Remark}
\newtheorem{exa}{Example}
\theoremstyle{definition}
\newtheorem{definition}{Definition}

\newcommand{\dom}{\operatorname{dom}}

\newcommand{\R}{\mathbb{R}}
\newcommand{\inner}[2]{\langle #1,#2\rangle}
\newcommand{\Rp}{\R_{\geqslant 0}}
\newcommand{\Z}{\mathbb{Z}}
\newcommand{\Zp}{\mathbb{N}}

\newcommand{\bba}{\bm{\alpha}}
\newcommand{\bbz}{\mathbf{z}}
\newcommand{\bbx}{\mathbf{x}}

\newcommand{\lse}{\mathrm{LSE}}
\newcommand{\lset}{\lse_{T}}
\newcommand{\Rpp}{\R_{>0}}
\newcommand{\mD}{\mathcal{C}}
\newcommand{\mLD}{\mathcal{L}}
\newcommand{\bbu}{\mathbf{u}}

\newcommand{\bby}{\mathbf{y}}
\newcommand{\bbv}{\mathbf{v}}

\newcommand{\eps}{\varepsilon}
\newcommand{\pos}{\mathrm{POS}}
\newcommand{\gpos}{\mathrm{GPOS}_T}
\newcommand{\bbq}{\mathbf{q}}

\newcommand{\bbb}{\bm{\beta}}

\newcommand{\mN}{\mathcal{N}}

\newcommand{\ma}{\mathrm{MA}}
\newcommand{\mS}{\mathcal{S}}
\newcommand{\ffnn}{\mathrm{FFNN}}

\newcommand{\bbxi}{\bm{\xi}}
\newcommand{\tran}{^{\top}}
\newcommand{\lseffnn}{\mathrm{LSE\text{-}FFNN}}

\newcommand{\beq}{\begin{equation}}
\newcommand{\eeq}{\end{equation}}

\usepackage[colorinlistoftodos,bordercolor=orange,backgroundcolor=orange!20,linecolor=orange,textsize=scriptsize]{todonotes}

\title{Log-sum-exp neural networks and posynomial models for  convex and log-log-convex data}

\author{\thanks{G. C. Calafiore and C. Possieri are with the Dipartimento di Elettronica e Telecomunicazioni, Politecnico di Torino, 10129 Turin, Italy
(e-mail: giuseppe.calafiore@polito.it, corrado.possieri@polito.it). \newline
\indent G. C. Calafiore is also associated with IEIIT-CNR Torino, Italy.\newline 
\indent S. Gaubert is with INRIA and CMAP, Ecole polytechnique, UMR 7641 CNRS, France (e-mail: stephane.gaubert@inria.fr). 
}
Giuseppe C. Calafiore, \IEEEmembership{Fellow,~IEEE}, Stephane Gaubert, \IEEEmembership{Member,~IEEE} \\and Corrado Possieri, \IEEEmembership{Associate Member,~IEEE}}

\begin{document}

\maketitle

\begin{abstract}
We show in this paper that a one-layer feedforward neural network with  exponential activation functions in the inner layer and logarithmic activation in the output neuron is an universal approximator of convex functions. Such a network represents a family of scaled log-sum exponential functions,  here named $\lset$. Under a suitable exponential transformation, the class of
$\lset$ functions maps to a family of generalized posynomials $\gpos$, 
which we similarly show to be universal approximators for log-log-convex functions.

A key feature of an $\lset$ network is that, once it  is trained on data, the resulting model is {\em convex} in the variables, which makes it readily amenable to efficient design based on convex optimization. Similarly, once a $\gpos$ model is trained on data, it yields a posynomial model that can be efficiently optimized with respect to its variables by using geometric programming (GP).

The proposed methodology is illustrated  by two numerical examples, 
in which, first, models are constructed from simulation data
of the two physical processes   
(namely, the level of vibration in a vehicle suspension system, and the peak power generated by the  combustion of propane), and then 
optimization-based design is performed on  these models.
\end{abstract}


\section{Introduction}
\subsection{Motivation and context}
A key challenge that has to be faced when dealing with real-word engineering analysis and design problems
is to find a model for a process or apparatus that is able to correctly interpret the observed data. 
The advantages of having at
one's disposal a mathematical model include enabling the analysis of extreme situations, 
the verification of decisions, the avoidance of time-consuming and expensive experimental tests or intensive numerical simulations, 
and the  possibility of optimizing over model parameters for the purpose of design.  
In this context, a tradeoff must be typically made between the accuracy of the model (here broadly intended as the capacity of the model in reproducing the experimental or simulation data) and its complexity, insofar as the former usually increases with the complexity of the model.  Actually, the use of  ``simple'' models of complex fenomena is gaining increasing interest in engineering design; examples are the so-called {\em surrogate models} constructed from complex simulation data arising, for instance, in aerodynamics modeling, see, e.g., \cite{YoAnVa:18, forrester2008engineering, gorissen2010surrogate}.

In particular, if the purpose of the model is performing optimization-based design, then it becomes of paramount importance to have a model that is suitably tailored for optimization. To this purpose, it is well known that an extremely advantageous property for a model to possess is {\em convexity}, see, e.g., \cite{calafiore2014optimization, boyd2004convex}.
In fact, if the objective and  constraints in an optimization-based design problem are convex, then efficient tools
(such as interior-point methods, see, e.g., \cite{potra2000interior}) can be used to solve the problem in an efficient, global and guaranteed sense. 
Conversely,  finding the solution to a generic nonlinear programming problem may be extremely difficult,
involving  compromises such as  long computation time or suboptimality of the solution; \cite{boyd2004convex,rockafellar1993lagrange}. 
Clearly, not all real-world models are convex, but several relevant ones are indeed convex, or can anyways be {\em approximated} by convex ones. In all such cases it is  of critical importance to be able to construct convex models from the available data.

The focus of this work is on the construction of functional models from data, possessing the desirable property of convexity.
Several tools have been proposed in the literature to fit data via  convex or log-log-convex functions
(see Section~\ref{subsec:posy} for a definition of log-log convexity).
Some remarkable examples are, for instance, \cite{magnani2009convex}, where an efficient least-squares partition algorithm 
is proposed to fit data through max-affine functions; \cite{kim2010convex}, where a similar method  has been
proposed to fit max-monomial functions; \cite{hoburg2016data}, where a technique based on 
fitting the data through implicit softmax-affine functions has been proposed; and 
\cite{daems2003simulation,calafiore2015sparse}, 
where methods to fit data through posynomial models have been proposed.

\subsection{Contributions}
Since the pioneering works \cite{cybenko1989approximation,white1990connectionist,HORNIK1991251}, artificial feedforward neural networks
have been widely used to find models apt at describing the data, see, e.g.,   \cite{ruck1990multilayer,vt1994radial,andras2014function}.
However, the input-output map represented by a neural network need not possess properties such as convexity, and hence
the ensuing model is in general unsuitable for optimization-based design.

The main objective of this paper is to show that, if the activation function of the hidden layer and of the output layer are 
properly chosen, then it is possible to design a feedforward neural network with one hidden layer that fits the data and that 
represents a  convex function of the inputs.
 Such a goal is pursued by studying the properties of the
log-sum-exp (or softmax-affine)  $\lset$ class of functions, by showing that they can be represented through a feedforward  neural network, and by proving that they posses universal approximator properties with respect to convex functions;  this constitutes our main result, stated in \Cref{cor:mainresult} and specialized in  \Cref{cor:convAppr} and \Cref{cor:convApprFFNN}. 
Furthermore, 
we show that an exponential transformation maps the class of $\lset$ functions
into the generalized posynomial family  $\gpos$, which can be used for fitting log-log convex data, as stated in \Cref{cor:mainresult2}, \Cref{cor:llconvappr}, and 
\Cref{cor:convApprFFNN_GPOS}.
Our approximation proofs rely in part on {\em tropical} techniques.
The application of tropical geometry to neural networks
is an emerging topic --- two recent works have used tropical methods
to provide combinatorial estimates, in terms of Newton polytopes,
of the ``classifying power'' of neural networks
with piecewise affine functions, see~\cite{charisopoulos2017morphological, lim}. Although there is no direct relation with the present results,
a comparison of these three works does suggest tropical methods
may be of further interest in the learning-theoretic context.

We flank the theoretical results in this paper with  a
numerical \texttt{Matlab} toolbox, named  \texttt{Convex\_Neural\_Network}, which we developed and made freely available on the web\footnote{See 
\url{https://github.com/Corrado-possieri/convex-neural-network/}}. This toolbox implements the proposed class
of feedforward neural networks, and
it has been used for the numerical experiments reported in the examples section.  

Convex neural networks are important in engineering applications in the context of construction of { surrogate models} for describing and optimizing complex input-output relations. We provide  examples of application
to two complex physical processes:
the amount of vibration transmitted  by a vehicle suspension system as a function of its mechanical parameters,
and the peak power generated by the combustion reaction  of propane as a function of the initial concentrations of the involved
chemical species. 

\subsection{Organization of the paper}
The remainder of this paper is organized as follows: in Section~\ref{sec:notation} we introduce the notation and we give some preliminary results about the classes of functions under consideration.
In Section~\ref{sec:approx}, we illustrate the approximation capabilities of the considered classes of functions, by establishing
that generalized log-sum-exp functions and generalized posynomials are universal smooth approximators of convex and
log-log-convex data, respectively. 
In Section~\ref{sec:algo}, we show the correspondence between these functions and
feedforward neural networks with properly chosen activation function. The effectiveness of the proposed approximation 
technique in realistic applications is highlighted in Section~\ref{sec:appli}, where the $\lset$ class is used to
perform data-driven optimization of two physical phenomena. Conclusions are given in Section~\ref{sec:concl}.

\section{Notation and technical preliminaries\label{sec:notation}}

Let $\Zp$, $\Z$, $\R$, $\Rp$, and $\Rpp$ denote the set of natural, integer, real, nonnegative real,
and positive real numbers, respectively.
{Given $\bbxi\in\R^n$, $\delta_{\bbxi}$ denotes the Dirac measure on the set $\{\bbxi\}\subset \R^n$.
The vectors $\bbxi_0,\dots,\bbxi_k\in\R^n$ are \emph{linearly independent} if 
$c_0\,\bbxi_0+\dots+c_k\,\bbxi_k\neq\bm{0}$
for all $c_0,\dots,c_k\in\R$ not identically zero, whereas they are 
\emph{affinely independent} if $\bbxi_1-\bbxi_0,\dots,\bbxi_k-\bbxi_0$ are linearly independent.}
Given $f:\R^n\rightarrow\R\cup\{+\infty\}$, 
{let \[\dom f\doteq\{\bbx\in\R^n:f(\bbx)<+\infty \}.\]
Supposing that
$\dom f\neq \emptyset$,} define the \emph{Fenchel transform} 
$f^\star:\R^n\rightarrow\R\cup\{+\infty\}$ of $f$ as 
\[f^\star(\bbx^\star)=\sup_{\bbx\in\R^n}(\inner{\bbx^\star}{\bbx}-f(\bbx)),\]
where $\inner{\bbx}{\bby}$ denotes an inner product; in particular, the standard inner product
 $\inner{\bbx}{\bby}\doteq\bbx\tran\bby$ will be assumed all throughout this paper. 
By the Fenchel-Moreau theorem, \cite{borwein2010convex}, it results that $f=f{}^{\star}{}^{\star}$ if and only if $f$ is convex and lower semicontinuous,
whereas, in general, it holds that $f\geqslant f{}^{\star}{}^{\star}$. We shall assume henceforth
that all the considered convex functions are {\em proper}, 
meaning that
their domain is nonempty.

\subsection{The Log-Sum-Exp class of functions\label{subsec:lse}}
Let $\lse$ (Log-Sum-Exp) be the class of  functions $f:\R^n\rightarrow\R$ that can be written as
\begin{equation}
f(\bbx) = \log \left(\sum_{k=1}^K b_k \exp ( \inner{\bba^{(k)}}{\bbx} )\right),
\label{eq:lse_lse_reg}
\end{equation}
for some $K\in\Zp$, $b_k\in\Rpp$, $\bba^{(k)}=[\begin{array}{ccc}
\alpha_1^{(k)} & \cdots & \alpha_n^{(k)}
\end{array}]^\top\in\R^n$, $k=1,\dots,K$, where
 $\bbx=[\begin{array}{ccc}
x_1 & \cdots & x_n
\end{array}]^\top$ is a vector of variables. Further, given $T\in\Rpp$ (usually referred to as the
\emph{temperature}), define the
class $\lset$ of  functions $f_T:\R^n\rightarrow\R$ that can be written as
\begin{equation}
f_T(\bbx) = T \log \left(\sum_{k=1}^K b_k^{1/T} \exp ( \inner{\bba^{(k)}}{\bbx/T} )\right),
\label{eq:lse_lse_reg_trop}
\end{equation}
for some $K\in\Zp$, $b_k\in\Rpp$, and $\bba^{(k)}\in\R^n$, $k=1,\dots,K$. 
By letting
$\beta_k \doteq \log b_k$, $ k=1,\ldots,K$,
we have that functions in the family  $\lset$ can be equivalently parameterized as
\begin{equation}
f_T(\bbx) = T \log \left(\sum_{k=1}^K \exp ( \inner{\bba^{(k)}}{\bbx/T} + \beta_k/T )\right),
\label{eq:lse_lse_reg_trop_exp}
\end{equation}
where the $\beta_k$s have no sign restrictions.
It may sometimes be convenient to highlight the full parameterization 
of $f_T$, in which case we shall write $f_T^{(\overrightarrow{\bba},\bbb)}$,
where $\overrightarrow{\bba} = (\bba^{(1)},\ldots,\bba^{(K)})$, and
$\bbb =  (\beta_1,\ldots,\beta_K)$.
It can then be observed that,  for any $T>0$, the following property holds:
\beq
f_T^{(\overrightarrow{\bba},\bbb)} (\bbx) =  T f_1^{(\overrightarrow{\bba},\bbb/T)} (\bbx/T).
\label{eq:scaling}
\eeq
A key fact is  that each $f_T\in\lset$ is smooth and convex. 
Indeed, letting $\mu$ be a positive
Borel measure on $\R^n$,  following the terminology
of \cite{klartag2012centroid}, the {\em log-Laplace transform}
of $\mu$~is 
\begin{equation}\label{eq:loglaplacetransform}
M(\bbx) \doteq \log\left( \frac{1}{\mu(\R^n)}   \int_{\R^n} \exp(\inner{\bm{\tau}}{\bbx})\,\mathrm{d}\mu(\bm{\tau}) \right).
\end{equation}
The convexity of this function is well known, 
being a direct consequence of H\"older's inequality.
Hence, letting 	$\mu = \sum_{k=1}^K b_k \,\delta_{\bba^{(k)}}$ be a sum of Dirac measures,
we obtain that each $f\in\lse$ is convex.
The convexity of all $f_T\in\lset$ follows immediately 
by the fact that convexity is preserved under positive scaling. On the other hand, the smoothness of each 
$f_T\in\lset$ follows by the smoothness of the functions $\exp(\cdot)$ and $\log(\cdot)$ in their domain.
The interest in this class of functions arises from the fact that, as established in the subsequent~\Cref{cor:mainresult},
functions in $\lset$ are universal smooth approximators of convex functions.

In the following proposition, we show that if 
the points with coordinates
$\bba^{(1)},\dots,\bba^{(K)}$ constitute an affine generating family of $\R^n$,
or, equivalently, if one
can extract $n+1$ affinely independent vectors from $\bba^{(1)},\dots,\bba^{(K)}$ then the function $f_T(\bbx)$ given in \eqref{eq:lse_lse_reg_trop} is strictly convex.
In dimension $2$, this condition means that the family of points
of coordinates $\bba^{(1)},\dots,\bba^{(K)}$ contains the vertices of a triangle;
in dimension $3$,  the same family must contain the vertices of a tetraehedron, and so on.

\begin{prop}\label{prop:strictConv}
The function $f_T(\bbx)$ given in \eqref{eq:lse_lse_reg_trop}  is strictly convex whenever the vectors 
$\bba^{(1)},\dots,\bba^{(K)}$ constitute an affine generating family of $\R^n$.
\end{prop}

\begin{proof}
Let $\mu$ be a positive Borel measure on $\R^n$.
For every $\bbxi\in\R^n$, consider the random variable
$\mathbf{X}_{\bbxi}$, whose distribution $\nu_{\bbxi}$, absolutely continuous with respect
to $\mu$, has the Radon-Nikodym derivative $\frac{\mathrm{d}\nu_{\bbxi}}{\mathrm{d}\mu}$ 
equal to $\bbx \mapsto \exp(\inner{\bbxi}{\bbx})$. It can be checked that the Hessian of the log-Laplace
transform of $\mu$ is 
$\nabla^2M(\bbxi) = \mathrm{Cov}(\mathbf{X}_{\bbxi})$,
where $\mathrm{Cov}(\cdot)$ denotes the covariance matrix of the random variable at argument, see the proof of \cite[Prop~7.2.1]{brazitikos}.
Hence, as soon as the support of the distribution of $\mathbf{X}_{\bbxi}$
contains $n+1$ affinely independent points, this
covariance matrix is positive definite, which entails
the strict convexity of $M$. The proposition follows
by considering the log-Laplace transform of
$\mu=\sum_{k=1}^K b_k\, \delta_{\bba^{(k)}/T}$,
in which the support of $\mu$ is $\{\bba^{(1)}/T,\dots,\bba^{(K)}/T\}$.
\end{proof}

\begin{rem}\label{rk-affine}
If the points with coordinates $\bba^{(1)},\dots,\bba^{(K)}$ do not constitute an affine generating family of $\R^n$, we can find a vector $\bbu\in\R^n$ such that
$\langle \bbu , \bba^{(k)}-\bba^{(1)}\rangle =0$ for $k=2,\dots,K$. It follows
that 
\begin{align}
f_T(\bbx + s \bbu) = s \langle\bba^{(1)},\bbu \rangle + f_T(\bbx),\qquad \forall s\in\R,
\end{align}
showing that $f_T$ is affine in the direction $\bbu$. 
\end{rem}

We next observe that the function class $\lset$ enlarges as $T$ decreases, as stated more precisely in the following lemma. 
\begin{lem}\label{lem:nested}
For all  $T>0$ and each $p\in\Zp$, $p\geqslant 1$, one has
\[\lset\subset \lse_{T/p}.\]
\end{lem}

\begin{proof} By definition, for a function $f_T \in \lset$ there exist
$T>0$, $b_k >0$ and $\bba^{(k)}$, $k=1,\ldots,K$, such that
\begin{align*}
f_T(\bbx) &= T \log \left(\sum_{k=1}^K b_k^{1/T} \exp ( \inner{\bba^{(k)}}{\bbx/T} )\right) 
\\
&\hspace{-3ex}= (T/p) \log \left( \sum_{k=1}^K b_k^{1/T} \exp (\inner{\bba^{(k)}}{\bbx/T} ) \right)^p 
\\&\hspace{-3ex}=  (T/p) \log \left( \sum_{k=1}^K (b_k^{1/p})^{p/T} \exp (\inner{\bba^{(k)}/p}{\bbx/(T/p)}) \right)^p  \\
&\hspace{-3ex}=  (T/p) \log  \left(\sum_{k=1}^{K'} \tilde b_k^{p/T} \exp ( \inner{\tilde{\bba}^{(k)}}{\bbx/(T/p)})\right),
\end{align*}
where the last equality follows from the observation that, by expanding the  (integer) power $p$, we obtain a summation over $K'\geqslant K$ terms,
each of which has the form of products of terms taken from the larger parentheses. 
These terms retain the  format
of the original terms in the parentheses, only with suitably modified parameters $\tilde b_k$ and $\tilde \bba^{(k)}$. 
The claim then follows by observing that the last expression represents a function in $\lse_{T/p}$.
\end{proof}

Consider now the class $\ma$ of \emph{max-affine functions} with $K$ terms,
i.e., the class of all the functions that can be written as
\begin{equation}
\label{eq:maxaffine}
\bar{f}(\bbx)\doteq \max_{k=1,\dots,K} ( \beta_k+\inner{\bba^{(k)}}{\bbx} ).
\end{equation} 
When the entries of $\bba^{(k)}$ are nonnegative integers, the function $\bar{f}$
is called a {\em tropical polynomial},~\cite{viro,itenberg}. Allowing these entries to be relative
integers yields the class of {\em Laurent tropical polynomials}.
When these entries are real, by analogy with classical
posynomials (see Section~\ref{subsec:posy}), the function $\bar{f}(\bbx)$ is sometimes referred to as a \emph{tropical posynomial}.
Note that the class of $\ma$ functions has been recently used in learning problems, \cite{charisopoulos2017morphological}, 
\cite{lim}, and in data fitting, see \cite{magnani2009convex}
and \cite{hoburg2016data}.
Such functions are convex, since the function obtained by taking the point-wise maximum of convex functions is convex.
It follows  from the parameterization in (\ref {eq:lse_lse_reg_trop_exp})  that, for all $\bbx\in\R^n$,
$\lim_{T \searrow 0}f_T(\bbx)=\bar{f}(\bbx)$,
i.e., the function $f_T$ given in \eqref{eq:lse_lse_reg_trop_exp}  approximates
$\bar{f}$ as $T$ tends to zero, see \cite{hoburg2016data}.
This deformation 
is familiar in tropical geometry under the name of ``Maslov dequantization,''~\cite{litvinov},
and it is a key ingredient of Viro's patchworking method,~\cite{viro}. 
The following uniform bounds are rather standard, but their formal proof is given here for completeness.



\begin{lem}\label{prop:approx}
For any $T\in\Rpp$, $f_T$ in \eqref{eq:lse_lse_reg_trop_exp},  and for all $\bbx\in\R^n$, it holds that
\begin{equation}
\bar f (\bbx)  \leqslant  f_T(\bbx) \leqslant T\log K + \bar f (\bbx).
\label{eq:metric_estimate}
\end{equation}
\end{lem}

\begin{proof}
By construction, we have that
\begin{align*}
\bar{f}(\bbx) & = \max_{k=1,\dots,K} ( \beta_k+\inner{\bba^{(k)}}{\bbx} )\\
& =\max_{k=1,\dots,K}T\log((\exp (\beta_k+ \inner{\bba^{(k)}}{\bbx} ))^{1/T})\\
& = T\log\left(\max_{k=1,\dots,K}(\exp (\beta_k+ \inner{\bba^{(k)}}{\bbx }))^{1/T}\right)\\
& \leqslant T\log\left(\sum_{k=1}^K (\exp (\beta_k+ \inner{\bba^{(k)}}{\bbx} ))^{1/T}\right)\\
& = f_T(\bbx),
\end{align*}
thus proving the left-hand side of the inequality in \eqref{eq:metric_estimate}. On the other hand,
we have that
\begin{align*}
f_T(\bbx) & = T\log\left(\sum_{k=1}^K\exp (\beta_k/T+ \inner{\bba^{(k)}}{\bbx/T} )\right)\\
&  \leqslant T\log\left( K\left(\exp \left(\max_{k=1,\dots,K}(\beta_k+ \inner{\bba^{(k)}}{\bbx} )\right)\right)^{1/T}\right)\\
& =  T\log( K(\exp(\bar{f}(\bbx))^{1/T}))\\
&= T\log K + \bar f (\bbx),
\end{align*}
thus proving the right-hand side of the inequality in \eqref{eq:metric_estimate}.
\end{proof}

\subsection{Posynomials\label{subsec:posy}}

Given $c_k\in\Rpp$ and $\bba^{(k)}\in\R^n$, a \emph{positive monomial}
is a product of the form $c_k\bbx^{\bba^{(k)}} = c_k x_1^{\alpha_1^{(k)}}x_2^{\alpha_2^{(k)}}\cdots x_n^{\alpha_n^{(k)}}$.
A \emph{posynomial} is a finite sum of positive monomials,
\begin{equation}
\psi(\bbx) = \sum_{k=1}^K  c_k \bbx^{\bba^{(k)}}.
\label{eq:POS}
\end{equation}
Posynomials are thus functions $\psi:\Rpp^n\rightarrow\Rpp$; we let $\pos$ denote the class of all posynomial functions.

\begin{definition}[Log-log-convex function]
A function   $\varphi(\bbx):\Rpp^n\rightarrow\Rpp$  is log-log-convex
if $\log \varphi$ is convex in $\log (\bbx)$.
\end{definition}

A positive monomial function $\varphi_k(\bbx) \doteq c_k\bbx^{\bba^{(k)}}$ 
 is clearly \emph{log-log-convex}, since $\log \varphi_k(\bbx)$ is linear (hence convex)
 in $\log \bbx$.
 Log-log convexity of functions in the $\pos$ family 
can be derived from the following proposition, which goes back to Kingman,~\cite{Kin61}.
  
\begin{prop}[Lemma p.~283 of~\cite{Kin61}]
\label{prop:loglogprop}
If $f_1(\bbx)$ and $f_2(\bbx)$ are log-log-convex functions, then the following functions are log-log-convex:
\begin{enumerate}[i)]
\item $\varphi_a(\bbx) = f_1(\bbx)+f_2(\bbx)$,
\item  $\varphi_b(\bbx) = f_1(\bbx)f_2(\bbx)$,
\item  $\varphi_c(\bbx) = \max(f_1(\bbx), f_2(\bbx))$,
\item $\varphi_d(\bbx) = f_1(\bbx)^p$, $p\in\Rpp$.
\end{enumerate}
\end{prop}

\if{\begin{proof}
By definition, $f_1(\bbx)$ and $f_2(\bbx)$ are log-log-convex in $\bbx$ if and only if $f_1(\exp(\bbu))$ and $f_2(\exp(\bbu))$ 
are log-convex in $\bbu$. Since the sum of log-convex functions is log-convex (see, e.g., Section~3.5.2 in \cite{boyd2004convex}), 
we have that \[\varphi_a(\exp(\bbu)) = f_1\exp(\bbu)+f_2\exp(\bbu)\] is log-convex, whence $\varphi_a(\bbx)$ is log-log-convex. 
Function \[\log \varphi_b (\exp(\bbu))= \log f_1(\exp(\bbu)) + \log f_2(\exp(\bbu))\] is convex since it is the sum of convex functions, 
whence $\varphi_b(\bbx)$ is log-log-convex. Similarly,
\[\log \varphi_d (\exp(\bbu))= p\log f_1(\exp(\bbu))\] is the positive multiple of a convex function, hence convex. 
Finally, due to the fact that $\log$ is monotonically increasing,
\begin{multline*}
\log \varphi_c(\exp(\bbu))= \log \max \left( f_1(\exp(\bbu)), f_2(\exp(\bbu)) \right) \\
= \max \left( \log f_1(\exp(\bbu)), \log f_2(\exp(\bbu)) \right),
\end{multline*}
which is convex, since the point-wise maximum of convex functions is convex.
\end{proof}}\fi
Since $c_k\bbx^{\bba^{(k)}}$ is log-log-convex, then by  \Cref{prop:loglogprop} 
each function in the $\pos$ class  is log-log-convex.
Posynomials are of great interest in practical applications since, under a log-log transform, they become
convex functions \cite{hoburg2016data,boyd2007tutorial}. More precisely, by letting $\bbq \doteq \log \bbx$, 
one has that
\begin{equation*}
\log \left(\sum_{k=1}^K  c_k \bbx^{\bba^{(k)}} \right)= \log \left(\sum_{k=1}^K  c_k \exp (\inner{\bba^{(k)}}{\bbq})\right),
\end{equation*}
which is a function in the $\lse$ family. 
Furthermore, given $T\in\Rpp$, since positive scaling preserves
convexity, \cite{boyd2004convex}, letting $\psi$ be a posynomial,
we have that functions of the form
\begin{equation}
\psi_T(\bbx) = (\psi(\bbx^{1/T}))^T
\label{eq:genposy}
\end{equation}
are log-log-convex. 
Functions that can be rewritten in the form~\eqref{eq:genposy}, with 
$\psi\in\pos$, are here denoted by $\gpos$ and they form a subset of  the family of the
so-called generalized posynomials.
It is a direct consequence of the above discussion that  $\lset$ and $\gpos$ functions are related by a one-to-one correspondence, as stated in the following proposition.

\begin{prop}
\label{prop:lse-gposmapping}
Let $f(\bbx)\in\lset$ and $\psi(\bbz)\in \gpos$. Then,
\begin{align*}
\exp \left( f\left( \log(\bbz)\right)  \right) & \in \gpos, \\
\log \left( \psi \left( \exp(\bbx)\right)  \right)& \in \lset.
\end{align*}
\end{prop}

\section{Data approximation  via  $\lset$ \\ and $\gpos$ functions\label{sec:approx}}

The main objective of this section is to show that the classes $\lset$ and $\gpos$ can be
used to approximate convex and log-log-convex data, respectively. 
In particular, in Section~\ref{subsec:approxconv},
we establish that functions in $\lset$  are universal smooth approximators of convex 
data. Similarly, in Section~\ref{subsec:posyappr}, we show that functions in $\gpos$ are universal 
smooth approximators of log-log-convex data.

\subsection{Approximation of convex data via $\lset$\label{subsec:approxconv}}
Consider a collection $\mD$ of $m$ data pairs,
\begin{equation*}
\mD = \{(\bbx_1,y_1),\dots,(\bbx_m,y_m) \},
\end{equation*}
where $\bbx_i\in\R^n$, $y_i\in\R$, $i=1,\dots,m$, with
\begin{equation*}
y_i = g(\bbx_i) 
,\quad i=1,\ldots,m,
\end{equation*}
and where $g:\R^n\rightarrow \R$ is an unknown convex function. 
The data in $\mD$ are referred to as \emph{convex data}.
The main goal of this section is to show that there exists a function $f_T\in\lset$ that fits 
such convex data with arbitrarily small absolute approximation error.

The question of the uniform approximation of a convex
function by functions $f_T\in \lset$ can be considered
either on $\R^n$, or on compact subsets of $\R^n$. 
The latter situation is the most relevant to the approximation of finite
data sets.  It turns out that there is a 
general characterization of the class of functions
uniformly approximable over $\R^n$, which we state as 
\Cref{thm:convAppr}. 
We then derive an uniform approximation result
over compact sets (\Cref{cor:mainresult}). 
However, the approximation issue over the whole $\R^n$ has an intrinsic
interest.

\begin{thm}\label{thm:convAppr}
The following statements are equivalent.
\begin{enumerate}[(a)]
\item\label{it-3} The function $g:\R^n\rightarrow\R$ is convex
and $\dom g^\star \doteq  \{\bbu\in\R^n : g^\star(\bbu)<\infty\}$ is a polytope.
\item\label{it-2} For all $\eps\in\Rpp$, there is $\bar{T}\in\Rpp$ such that, $\forall T\in\Rpp$, $T\leqslant\bar{T}$,
there is $f_T\in \lset$ such that $\|f_T-g\|_\infty\leqslant \eps$.
\item\label{it-0} For all $\eps\in\Rpp$, there exists a convex polyhedral function $h$ such that $\|h-g\|_\infty\leqslant \eps$.
\end{enumerate}
\end{thm}

\begin{proof}
\eqref{it-2}$\implies$\eqref{it-3}:
If $\|f_T-g\|_{\infty}\leqslant \eps$ for some $\eps\in\Rpp$,  we have $\dom f_T^\star=\dom g^\star$.  Therefore, item~\eqref{it-2}
together with the metric estimate~\eqref{eq:metric_estimate},
which gives $\|f_T- \bar f\|_\infty\leqslant T\log K$, implies that $\dom g^\star =\dom f_T^\star = \dom \bar{f}^\star= \operatorname{conv}\{\alpha^{(1)},\ldots,\alpha^{(K)}\}$.
Item \eqref{it-2}
implies that $g$ is the pointwise limit of a sequence of convex functions,
and so $g$ is convex. 

\eqref{it-3}$\implies$\eqref{it-0}: Suppose now that \eqref{it-3} holds. Let us triangulate the polytope $P\doteq \dom g^\star$ into
finitely many simplices of diameter at most $\omega\in\Rpp$. Let $V$ denote the collection of vertices of these simplices, and define
the function $h:\R^n\rightarrow\R$,
\[
h(\bbx) \doteq \sup_{\bbv\in V} (\inner{\bbv}{\bbx} - g^\star(\bbv) ) .
\]
Observe that $h$ is convex and polyhedral. Since $g$ is convex and finite
{(hence $g$ is continuous by \cite[Thm.~10.1]{rockafellar:1970})},
we have 
\begin{multline*}
g(\bbx) = g{^\star}{^\star}(\bbx)
= \sup_{\bby\in \R^n} (\inner{\bby}{\bbx} - g^\star(y) )\\
= \sup_{\bby\in P} (\inner{\bby}{\bbx} - g^\star(\bby) )
\geqslant  h(\bbx)
\end{multline*}
Moreover,
for all $\bbx\in \R^n$, the latter supremum is attained by a point $\bby\in P$,
which belongs to some simplex of the triangulation. Let $\bbv_1,\ldots,\bbv_{{n+1}}\in V$
denote the vertices of this simplex, so that $\bby=\sum_{i=1}^{{n+1}} \gamma_i \bbv_i$
where $\gamma_i\geqslant 0$, $i=1,\dots,m$, and $\sum_{i=1}^{{n+1}}\gamma_i =1$. 
Since $g$ is polyhedral, we know that $g^\star$, which is a convex function
taking finite values on a polyhedron, is continuous on this polyhedron \cite{rockafellar:1970}. 
So, $g^\star$ is uniformly continuous on $P=\dom g^\star$.
It follows that we can choose $\omega\in\Rpp$ such that $\max_{i}\|g^\star (\bby)-g^\star(\bbv_i)\|\leqslant \eps$, 
for all $\bby\in P$ included in a simplex with vertices $\bbv_1,\ldots,\bbv_{{n+1}}$ of the triangulation. 
Therefore, we have that 
\begin{multline*}
g(\bbx)   = \inner{\bby}{\bbx} - g^\star(\bby) 
\leqslant  \inner{\bby}{\bbx} - \sum_{i=1}^{{n+1}} \gamma_i (g^\star(\bbv_i)-\eps) \\
\leqslant \sum_{i=1}^{{n+1}} \gamma_i (\inner{\bbv_i}{\bbx} - g^\star(\bbv_i) )+\eps
 \leqslant h(\bbx)+\eps,
\end{multline*}
 which shows that~\eqref{it-0} holds.

\eqref{it-0}$\implies$\eqref{it-2}: any convex polyhedral function $h:\R^n\rightarrow\R$ can be rewritten
in the following form:
\[
h(\bbx)= \max_{k=1,\ldots,K} ( \log b_k + \inner{\bba^{(k)}}{\bbx}),
\]
for some $K\in\Zp$, $b_k\in\Rpp$, and $\bba^{(k)}$, $k=1,\ldots,K$.
By~\eqref{eq:metric_estimate}, for each $\eps\in\Rp$, there is $\bar{T}\in\Rpp$ such that, for each $T\in\Rpp$, $T\leqslant\bar{T}$,
the function $f_T$ given in \eqref{eq:lse_lse_reg_trop} satisfies
$\|h-f_T\|_\infty\leqslant \eps$. Hence, if $\|h-g\|_\infty\leqslant \eps$, then $\|g-f_T\|_\infty\leqslant  2\eps$,
thus concluding the proof.
\end{proof}

 \begin{rem}
The condition that the domain of $g^*$ is a polytope 
in \Cref{thm:convAppr} is rather restrictive. This entails
that the map $g$ is Lipschitz, with constant
$\sup_{\bbu\in \dom g^\star} \|\bbu\|$, where $\|\cdot\|$ is the Euclidean
norm. In contrast, not every
Lipschitz function has a polyhedral domain. For instance,
if $g(\bbx)=\|\bbx\|$, $\dom g^\star$ is the unit Euclidean ball.
However, the condition on the domain of $g^\star$ only involves
the behavior of $g$ ``at infinity''. \Cref{cor:mainresult} below shows that when
considering the approximation problems
over compact sets, the restriction 
to a polyhedral domain can be dispensed with.
\end{rem}


\begin{thm}[Universal approximators of convex functions]\label{cor:mainresult}
  Let $f$ be a real valued continuous convex function defined
  on a compact convex set ${\mathcal K} \subset \R^n$.
  Then,
   For all $\eps > 0$ there exist $T>0$ and a function $f_T \in \lse_T$
such that
\begin{align}
  |f_T(\bbx) - f(\bbx)| \leqslant \eps,\quad\text{ for all }\; \bbx\in {\mathcal K} .\label{e-defap}
\end{align}
\end{thm}
If~\eqref{e-defap} holds, then $f_T$ is an \emph{$\varepsilon$-approximation of $f$ on $\mathcal{K}$}.
\begin{proof}
  We first show that the  statement of the theorem holds
  under the additional assumptions that $f$ is $L$-Lipschitz continuous on $\mathcal{K}$ for some constant $L>0$ and that $\mathcal{K}$ has non-empty interior.
  Observe that there is a sequence $(\bbx_k)_{k\geqslant 1}$ of elements in the interior of $\mathcal{K}$ that is dense in $\mathcal{K}$ (for instance, we may consider the set of vectors in the interior of $\mathcal{K}$ that have rational coordinates, this set is denumerable, and so, by indexing its elements in an arbitrary way, we get a sequence that is dense in $\mathcal{K}$). In what follows,
  we shall identify $f:\mathcal{K}\to \R$ with the convex function $\R^n\to \R\cup\{+\infty\}$ that coincides with $f$ on $\mathcal{K}$ and takes the value
  $+\infty$ elsewhere. Recall in particular that
  the {\em subdifferential} of $f$ at a point $\bby\in \mathcal{K}$ is the
  set
  \[ \partial f(\bby)\doteq\{\bbv\in \R^n\mid f(\bbx)-f(\bby)\geqslant\langle \bbv, \bbx-\bby\rangle ,\quad
  \forall \bbx\in \mathcal{K}\},
  \]
  and that, by Theorem~23.4 of \cite{rockafellar:1970},
  $\partial f(\bby)$ is non-empty for all $\bby$ in the relative interior of the domain of $f$, i.e., here, in the interior of $\mathcal{K}$. It is also known that
  $\|\bbv\|\leqslant L$ for all $\bbv\in \dom f^\star$, and in particular for all $\bbv\in \partial f(\bbx)$ with $\bbx\in \mathcal{K}$ (Corollary~13.3.3 of~\cite{rockafellar:1970}). Let us now choose
  in an arbitrary way an element $\bbv_k \in \partial f(\bbx_k)$, for each $k\geqslant 1$, and consider
  the map $f_\jmath: \R^n\to \R$, 
  \[
  f_\jmath(\bbx)\doteq
\max_{1\leqslant k\leqslant \jmath} \Big(f(\bbx_k) + \langle \bbv_k, \bbx-\bbx_k\rangle \Big) .
  \]
  By definition of the subdifferential, we have $f(\bbx)\geqslant f_\jmath(\bbx)$
  for all $\bbx\in \mathcal{K}$, and by construction of $f_\jmath$, $f(\bbx_k)=f_\jmath(\bbx_k)$ for all $1\leqslant k\leqslant \jmath$, so the sequence $(f_\jmath)_{\jmath\geqslant 1}$ converges pointwise to $f$
  on the set $X\doteq\{\bbx_k\mid k\geqslant 1\}$. 
  Since $\|\bbv_k\|\leqslant L$,
  every map $\bbx\mapsto f(\bbx_k) + \langle \bbv_k, \bbx-\bbx_k\rangle$ is Lipschitz of
  constant $L$, and so, $f_\jmath$ is also Lipschitz of constant $L$.
  Hence, the sequence of maps $(f_\jmath)_{\jmath \geqslant 1}$ is equi-Lispchitz.
  A fortiori, it is equicontinuous. 
  Then, by
  the second theorem of Ascoli (Th\'eor\`eme T.2, XX, 3; 1 of \cite{schwartz}),
  the pointwise convergence of the sequence
  $(f_\jmath)_{\jmath \geqslant 1}$ to $f$ on the set $X$
  implies that the same sequence converges {\em uniformly} to $f$ on the closure of $X$, that is, on $\mathcal{K}$. In particular, for all $\varepsilon>0$, we can find an integer $\jmath$ such that
  \begin{align}
    \sup_{\bbx\in\mathcal{K}} |f(\bbx)-f_\jmath(\bbx)|\leqslant \varepsilon/2  .
    \label{e-intermediate}
    \end{align}
  Consider now
  \[
  f_T(\bbx)\doteq
  T\log \Big(\sum_{1\leqslant k\leqslant \jmath} \exp\big(f(\bbx_k)/T + \langle \bbv_k/T, \bbx-\bbx_k\rangle \big)\Big) .
  \]
  By \Cref{prop:approx}, choosing any $T>0$ such that $T\log \jmath\leqslant \varepsilon/2$ yields $|f_\jmath(\bbx)-f_T(\bbx)|\leqslant \varepsilon/2$ for all $\bbx \in \R^n$. Together with~\eqref{e-intermediate}, we get
  $|f(\bbx)-f_T(\bbx)|\leqslant \varepsilon $ for all $\bbx\in \mathcal{K}$,
  showing that the  statement of the theorem indeed holds.
  
We now relax the assumption that $f$ is Lipschitz continuous.
Consider,
for all $\eta>0$, the Moreau-Yoshida regularization of $f$, which
is the map $g_\eta: \R^n\to \R$ defined by 
\begin{align}
  g_\eta(\bbx)  =\inf_{\bby\in \mathcal{K}} \Big(
  \frac{1}{2\eta}\|\bbx-\bby\|^2 + f(\bby) 
  \Big)
   , \forall \bbx \in \R^n  .
\end{align}
Observe that $\eta\mapsto g_\eta$ is nonincreasing, and that
$g_\eta\leqslant g$. It is known that the function $g_\eta$ is convex, being
the inf-convolution of two convex functions (Theorem~5.4 of \cite{rockafellar:1970}), it is also known that
$g_\eta$ is Lipschitz of constant $1/(2\eta)$ (Th.~4.1.4, \cite{lemarechal}) and that the family of functions $(g_\eta)_{\eta>0}$
converges pointwise to $f$ as $\eta\to 0^+$ (Prop.~4.1.6, {\em ibid.}).
Moreover, we supposed that $f$ is continuous.
We now use a theorem of Dini, showing that if a nondecreasing family of
continuous real-valued maps defined on a compact set converges
pointwise to a continuous function, then this family converges {\em uniformly}.
It follows that $g_\eta$ converges {\em uniformly} to $f$ on the compact set $\mathcal{K}$
as $\eta \to 0^+$. In particular, we can find $\eta>0$ such that
$|f(\bbx)-g_\eta(\bbx)|\leqslant \varepsilon/2$ holds for all $\bbx\in \mathcal{K}$.
Applying the statement of the theorem, which is already proved in the case of Lipschitz convex maps, to the map $g_\eta$, we get that there exists a map $f_T\in \lse_T$ for some $T>0$ such that $|f_T(\bbx)-g_\eta(\bbx)|\leqslant \varepsilon/2$ holds for all $\bbx\in \mathcal{K}$,
and so $|f_T(\bbx)-f(\bbx)|\leqslant  \varepsilon$, for all $\bbx\in \mathcal{K}$,
showing that the  statement of the theorem again holds for $f$.

Finally, it is easy to relax the assumption that $\mathcal{K}$ has non-empty
interior: denoting by $E$ the affine space generated by $\mathcal{K}$, we
can decompose a vector $\bbx \in \R^n$ in an unique way as $\bbx = \bby + \bbz$
with $\bby \in E$ and $\bbz\in E^\top$, where $E^\top = \{\bbz \mid \langle \bbz, \bby-\bby'\rangle =0,\forall \bby,\bby'\in E\}$. Setting $\bar f(\bbx)\doteq f(\bbz)$
allows us to extend $f$ to a convex continuous function $\bar f$, constant on any direction orthogonal to $E$, and whose domain contains $\bar{\mathcal{K}}\doteq \{\bby + \bbz \mid \bby \in \mathcal{K}, \|\bbz\|\leqslant 1\}$ which is a compact convex set of non-empty interior. By applying the statement of the theorem to $\bar{f}$, we get a $\varepsilon$-approximation of $\bar{f}$ on $\bar{\mathcal{K}}$ by a map $f_T$ in $\lse_T$. A fortiori, $f_T$ is a $\varepsilon$-approximation of ${f}$ on ${\mathcal{K}}$.

  \end{proof}
\begin{rem}\label{rem:specCase}
A useful special case arises when $f$ is a convex function from $\R^n \to \R\cup\{+\infty\}$, and ${\mathcal K}$ is included in the relative interior of $\dom f$. Then, the continuity assumption in \Cref{cor:mainresult} is automatic, see
e.g., Theorem~10.4 of \cite{rockafellar:1970}.
\end{rem}

{The following proposition is now an immediate consequence of
\Cref{cor:mainresult}, where $\mathcal K$ can be taken as the convex hull of the
input data.

\begin{prop}[Universal approximators of convex data]\label{cor:convAppr}
Given a collection of convex data $\mD\doteq\{(\bbx_i,y_i)\}_{i=1}^m$
generated by an unknown convex function,
for each $\eps\in\Rpp$ there exists ${T} > 0$
and $f_T\in \lset$ such that 
\[
|f_T(\bbx_i)-y_i|\leqslant \eps,\quad i=1,\dots,m.
\]
\end{prop}


The following counterexample shows that, in general, we cannot
 find a function $f_T$ matching exactly the data points,
i.e., some approximation is sometimes unavoidable.

\begin{exa}\label{rem:nofit}
Suppose first that $n=1$, consider the function $\phi(x)=\max(0,x-1)$, and the data
$\bbx_1=1$, 
$\bbx_2=-1$, 
$\bbx_3=0$, 
$\bbx_4=2$, with $y_i=g (x_i)$ for $i= 1,\ldots,4$,
so $y_1=y_2=y_3=0$ and $y_4=1$.
Suppose now that this dataset is matched exactly by 
a function $f_T\in\lse_T$ with $T>0$, parametrized
as in~\eqref{eq:lse_lse_reg_trop_exp}.
Since the points $(\bbx_1,y_1), \dots, (\bbx_4,y_4)$
are not aligned, we know, by \Cref{rk-affine}, that  
the family $\{\bba^{(1)},\dots,\bba^{(k)}\}$ contains an affinely generating
family of $\R$ (in dimension $1$, this simply means that $\bba^{(i)}$ take at least two values). It follows from \Cref{prop:strictConv}
that $f_T$ is strictly convex. However, a strictly convex function cannot match exactly the subset of data
$(-1,0),(0,0),(1,0)$, as it
consists of three aligned points.

This entails that in any dimension $n\geqslant 2$, there are also data sets that
cannot be matched exactly. Indeed, if $f_T\in \lset$ is a function of $n$ variables, then, for any vectors $\bba,\bbu\in\R^n$, the function $\bar{f}_T: s\mapsto f_T(\bba  + s\bbu)$ of one variable is also in $\lse_T$. Hence, if any data set 
$(\bbx_i,y_i), i=1,\ldots, m$, is such that a subset of points $(\bbx_i)_{i\in I}$
is included in an affine line $L$, and if a function $f_T$ matches exactly the
set of data, then, the function $\bar{f}_T$ is the solution of an exact
matching problem by an univariate function in $\lset$, and the previous dimension $1$ counter example shows that this problem need 
not be solvable.
\end{exa}

\subsection{Approximation of log-log-convex data  via $\gpos$\label{subsec:posyappr}}
Consider a collection $\mLD$ of $m$ data pairs,
\begin{equation*}
\mLD = \{(\bbz_1,w_1),\dots,(\bbz_m,w_m) \},
\end{equation*}
where $\bbz_i\in\Rpp^n$, $w_i\in\Rpp$, $i=1,\dots,m$, with
\begin{equation*}
w_i = \ell(\bbz_i) 
,\quad i=1,\dots,m,
\end{equation*}
where $\ell:\Rpp^n\rightarrow \Rpp$ is an unknown log-log-convex function. 
The data in $\mLD$ is referred to as \emph{log-log-convex}.
The following corollary 
states that there exists $\psi_T\in\gpos$ that
fits the data $\mLD$ {with arbitrarily small relative approximation error}.
A subset ${\mathcal R}\subset \Rpp^n$ will be said to be {\em log-convex}
if its image by the map which performs the $\log$ entry-wise is convex.

\begin{cor}
[Universal approximators of log-log-convex functions]\label{cor:mainresult2}
Let $\ell$ be a log-log-convex function defined on a compact
log-convex subset
${\mathcal R}\subset \Rpp^n$. 
Then, for any $\tilde{\eps} > 0$ there exist $T>0$ and a function $\psi_T \in \gpos$ 
such that, for all  $\bbx\in\mathcal{R}$,
\begin{equation}\label{eq:relBoun}
\left\vert \frac{\ell(\bbx)-\psi_T(\bbx)}{\min(\ell(\bbx),\psi_T(\bbx))}\right\vert\leqslant \tilde{\eps}.
\end{equation}
\end{cor}

\begin{proof}
By using the log-log transformation, define $\tilde{\ell}(\bbq)\doteq\log(\ell(\exp(\bbq)))$. Since $\ell(\bbx)$ is
log-log-convex in $\bbx$, $\tilde{\ell}(\bbq)$ is convex in $\bbq=\log \bbx$. 
Furthermore, the set $\mathcal{K}\doteq \log(\mathcal{R})$
is convex and compact since the set $\mathcal{R}$ is log-convex and compact.
Thus, by \Cref{cor:mainresult}, for all ${\eps}\in\Rpp$, there exist $T>0$ and a function $f_T \in \lse_T$ such that $|f_T(\bbq)-\tilde{\ell}(\bbq)|\leqslant {\eps}$ for all $\bbq\in\mathcal{K}$. Note that, by construction 
\begin{align*}
 \exp(f_T(\bbq))
&=\exp\left(T\log\left(\sum_{k=1}^K \exp (\beta_k/T+ \inner{\bba^{(k)}}{\bbq/T} )\right)\right)\\
&= \left(\sum_{k=1}^K\exp (\beta_k/T+ \inner{\bba^{(k)}}{\log(\bbx^{1/T})} )\right)^T\\
&= \left(\sum_{k=1}^K c_k(\bbx^{1/T})^{\bba^{(k)}}\right)^T=\psi_T(\bbx),
\end{align*}
where $c_k\doteq \exp(\beta_k/T)=b_k^{1/T}$ and $\psi_T(\bbx)\in\gpos$.
Thus, since, by the reasoning given above, 
we have $\exp(\tilde{\ell}(\bbq(\bbx)))=\ell(\bbx)$ and
$\exp(f_T(\bbq(\bbx)))=\psi_T(\bbx)$, it results that
\begin{equation*}
\begin{array}{rl}
\ell(\bbx)-\psi_T(\bbx)&=
\exp(\tilde{\ell}(\bbq))-\exp(f_T(\bbq))\\
&=\ell(\bbx)(1-\exp(f_T(\bbq)-\tilde{\ell}(\bbq)))\\
& = \psi_T(\bbx)(\exp(\tilde{\ell}(\bbq)-f_T(\bbq))-1).
\end{array}
\end{equation*}
Thus, it results that, for all $\bbx\in\mathcal{R}$,  
\begin{align*}
\left\vert \frac{\ell(\bbx)-\psi_T(\bbx)}{\ell(\bbx)} \right\vert & \leqslant \sup_{\bbq\in\mathcal{K}} \vert 1-\exp(f_T(\bbq)-\tilde{\ell}(\bbq))\vert
\leqslant \tilde{\eps},\\
\left\vert \frac{\ell(\bbx)-\psi_T(\bbx)}{\psi_T(\bbx)} \right\vert& \leqslant \sup_{\bbq\in\mathcal{K}} \vert \exp(\tilde{\ell}(\bbq)-f_T(\bbq))-1\vert
\leqslant \tilde{\eps},
\end{align*}
where $\tilde{\eps}\doteq 1-\exp({\eps})$.
Hence, \eqref{eq:relBoun} holds since $\tilde{\eps}$ can be made arbitrarily small by letting ${\eps}$ be sufficiently small.
\end{proof}

{The following proposition is now an immediate consequence of
\Cref{cor:mainresult2}, where $\mathcal R$ can be taken as the log-convex hull of the
input data points\footnote{For given points $\bbz_1,\ldots,\bbz_m\in\Rpp^n$, we define their log-convex hull as the set of vectors $\bbz= \prod_{i=1}^m\bbz_i^{\xi_i}$, where
$\xi_i\in[0,1]$ for all $i$ and $\sum_{i=1}^m\xi_i = 1$ (all operations are here intended entry-wise).}.

\begin{prop}\label{cor:llconvappr}
Given a collection of log-log-convex data $\mLD\doteq\{(\bbz_i,w_i)\}_{i=1}^m$, for each $\tilde{\eps}\in\Rpp$ there exist ${T}\in\Rpp$ and a $\psi_T\in \gpos$ such~that \[\left\vert \frac{\psi_T(\bbz_i)-w_i}{\min(\psi_T(\bbz_i),w_i)} \right \vert\leqslant \tilde{\eps},\quad i=1,\dots,m.\]
\end{prop}

\begin{rem}
A reasoning analogous to the one used in Remark~\ref{rem:nofit} can be employed to
show that, given a collection $\mLD$ of log-log-convex data pairs, there need not exist $\psi_T\in\gpos$
that matches exactly the data in $\mLD$, for any $T>0$.
\end{rem}}

Propositions~\ref{cor:convAppr} and \ref{cor:llconvappr} establish that functions in $\lset$ and
$\gpos$ can be used as universal smooth approximators  of convex and log-log-convex data, respectively.
However, there is a difference between the type of approximation of these two classes of functions. 
As a matter of fact, given a collection of convex data $\mD=\{(\bbx_i,y_i) \}_{i=1}^m$, the class $\lset$ is such that
there exists $f_T\in\lset$ such that  the \emph{absolute error} between $f_T(\bbx_i)$ and $y_i$ can be made arbitrarily small,
provided that $T\in\Rpp$ is sufficiently small.
On the other hand, given a collection of log-log-convex data $\mLD=\{(\bbz_i,w_i) \}_{i=1}^m$
the class $\gpos$ is such that, given $\mLD=\{(\bbz_i,w_i)\}_{i=1}^m$, 
there exists $\psi_T\in\gpos$ such that the \emph{relative error}
between $\psi_T(\bbz_i)$ and $w_i$ can be made arbitrarily small, 
provided that $T\in\Rpp$ is sufficiently small.
Figure~\ref{fig:impl} summarizes the results that have been established in this section
through \Cref{cor:convAppr} and \ref{cor:llconvappr}.

\begin{figure}[htb]
\centering
\resizebox{0.4\textwidth}{!}{
\includegraphics{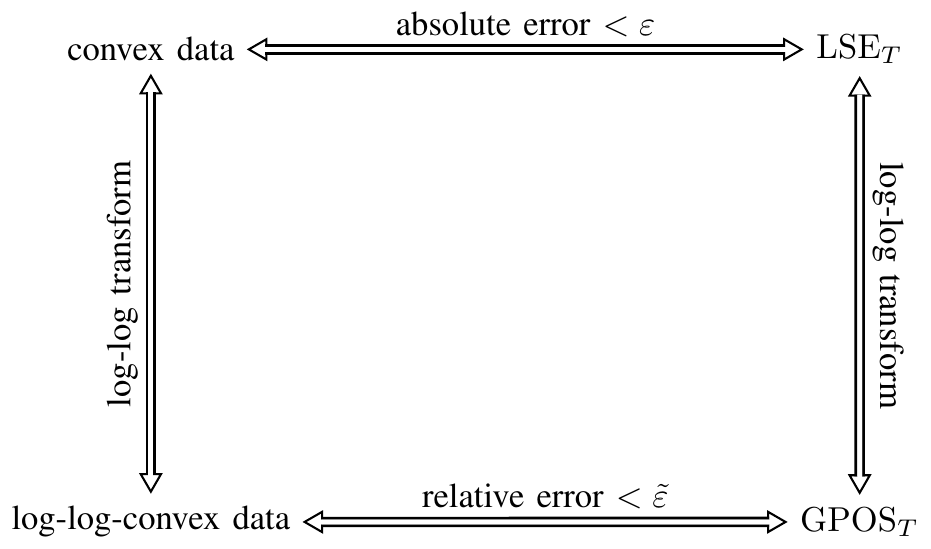}
}
\caption{Relation among the classes of functions and data.\label{fig:impl}}
\end{figure}

However, it is worth noticing that, since in any compact subset of $\Rpp^n$ bounding relative errors
is equivalent to bounding absolute errors, it follows that the class $\gpos$ is also such that
there exists $\psi_T\in\gpos$ such that the \emph{absolute error}
between $\psi_T(\bbz_i)$ and $w_i$ can be made arbitrarily small, provided that $T\in\Rpp$ is sufficiently small.

\section{Relation with feedforward neural networks\label{sec:algo}}

Functions in $\lset$ can
be modeled through a feedforward neural network (\emph{$\ffnn$}) with one hidden layer. 
Indeed, consider a $\ffnn$ 
with $n$ input nodes, one hidden layer with $K$ nodes, and one output node, as depicted in Figure~\ref{fig:appr}.

\begin{figure}[htb]
\centering
\resizebox{0.35\textwidth}{!}{
\includegraphics{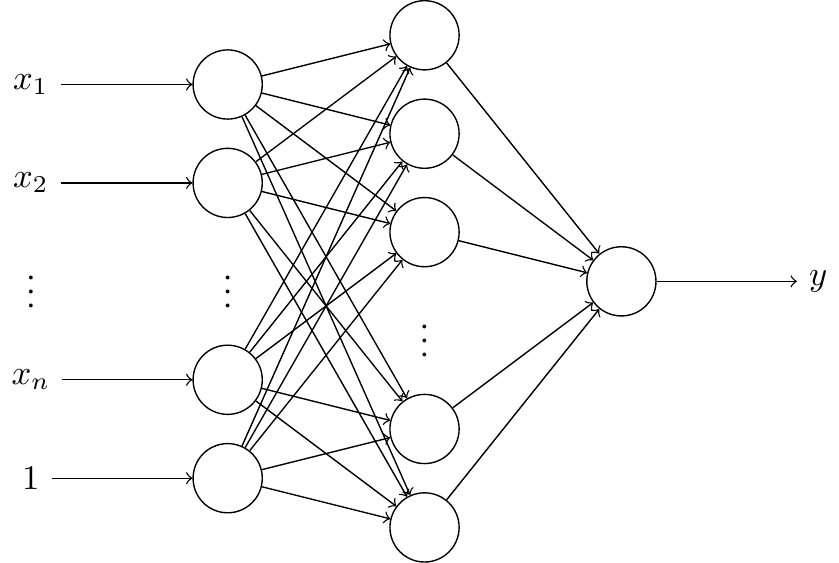}}
\caption{A feedforward neural network with one hidden layer. \label{fig:appr}}
\end{figure}

Let the activation  function of the hidden nodes be
 \[s\mapsto (\exp(s/T)),\]
 and let the activation of the output node be
 \[s\mapsto T\log(s).\]
Each node in the hidden layer  computes a term of the form  
$s_k =  \inner{\bba^{(k)}}{\bbx} +  \beta_k$,
where the $i$-th component $\alpha^{(k)}_i$ of $\bba^{(k)}$ represents the weight between node $k$ and input $x_i$, and $\beta_k$ is the bias term of node $k$. Each node $k$ thus generates activations
\[a_k = \exp (  \inner{\bba^{(k)}}{\bbx/T} +  \beta_k/T ).\]  
We consider the weights from the inner nodes to the output node to be unitary, whence the output node computes
$s = \sum_{k=1}^K a_k$ and then, according to the output activation function, the output layer returns
the value  
\[
y = T \log (s) = T \log \left(\sum_{k=1}^K a_k\right).
\]
We name such a network an $\lseffnn$.  
Comparing the expression of $y$ with (\ref{eq:lse_lse_reg_trop_exp}) it is readily seen that an
$\lseffnn$ allows us to represent any function
 in $\lset$.
We can then restate \Cref{cor:convAppr} as the  following key theorem.

\begin{thm}\label{cor:convApprFFNN}
Given a collection of convex data $\mD\doteq\{(\bbx_i,y_i)\}_{i=1}^m$
generated by an unknown convex function,
for each $\eps\in\Rpp$ there exists an $\lseffnn$ 
such that \[|f_T(\bbx_i)-y_i|\leqslant \eps,\quad i=1,\dots,m,\]
where $f_T$ is the input-output function of the $\lseffnn$.
\end{thm}

\Cref{cor:convApprFFNN} can be viewed as a specialization of the Universal Approximation Theorem \cite{cybenko1989approximation,white1990connectionist,HORNIK1991251,hornik1989multilayer} 
to convex functions. 
While universal-type approximation theorems provide theoretical approximation guarantees for general FFNN on general classes of functions, our \Cref{cor:convApprFFNN} only provides 
guarantees for data generated by convex functions. However, while general FFNN synthesize nonlinear and non-convex functions, $\lseffnn$ are guaranteed to provide a convex input-output map, and this is  a key feature of interest when the
synthesized model is to be used at a later stage as a basis for optimizing over the input variables.

$\lseffnn$s can also be used to fit log-log-convex data $\mLD =\{(\bbz_i,w_i)\}_{i=1}^m$:
by applying a log-log transformation $\bbx_i = \log \bbz_i$, $y_i = \log w_i$, 
$i=1,\ldots,m$, we simply transform log-log-convex data into convex data 
 $\mD =\{(\bbx_i,y_i)\}_{i=1}^m$ and train the network on these data. Therefore, the following theorem 
 is a direct consequence of \Cref{cor:convApprFFNN} and \Cref{cor:mainresult2}.

\begin{thm}\label{cor:convApprFFNN_GPOS}
Given a collection of log-log-convex data $\mLD\doteq\{(\bbz_i,w_i)\}_{i=1}^m$
generated by an unknown log-log-convex function,
for each $\tilde{\eps}\in\Rpp$ there exists an $\lseffnn$ 
such that 
\[\left|\frac{\exp(f_T(\log(\bbz_i)))-w_i}{\min(\exp(f_T(\log(\bbz_i))),w_i)}\right|\leqslant \tilde{\eps},\quad i=1,\dots,m,\]
where $f_T$ is the input-output function of the $\lseffnn$.
\end{thm}


\subsection{Implementation considerations}
Given training data $\mD =\{(\bbx_i,y_i)\}_{i=1}^m$, and for fixed $K$ and $T>0$,
the network weights $\bba^{(1)},\dots,\bba^{(K)},\beta_1,\ldots,\beta_K$ can be determined via standard training algorithms, such as
the Levenberg-Marquardt algorithm \cite{marquardt1963algorithm}, the gradient descent with momentum \cite{sutskever2013importance},
or the Fletcher-Powell conjugate gradient \cite{scales1985introduction}, which are, for instance, efficiently implemented in \texttt{Matlab}  
through the \texttt{Neural Network Toolbox} \cite{nntoolbox}.
These algorithms tune the network's weights in order to minimize 
a loss criterion  of the form 
\[L = \sum_{i=1}^m L_i(f_T(\bbx_i) -y_i ) + {R},
\]
 where the observation loss $L_i$ is typically a standard quadratic or an absolute value loss, and  ${ R}$ is a regularization term that does not depend on the training data.
For given network parameters $\overrightarrow{\bba}$ 
and $\bbb 
$, using \eqref{eq:scaling} we observe that
\begin{eqnarray*}
f_T^{(\overrightarrow{\bba},\bbb)}(\bbx_i) -y_i &=&   
T f_1^{(\overrightarrow{\bba},\bbb/T)}(\bbx_i/T) -y_i \\
&=& T \left(f_1^{(\overrightarrow{\bba},\bbb/T)}(\bbx_i/T) - y_i/T  \right).
\end{eqnarray*}
Hence, the loss term $\sum_{i=1}^m L_i(f_T(\bbx_i) -y_i )$ is proportional to
$\sum_{i=1}^m L_i(f_1(\bbx_i/T) -y_i/T )$ for the usual quadratic and absolute losses.
Thus, the temperature $T>0$ can be implemented in practice by pre-scaling the data
(i.e., divide the inputs and outputs by $T$), and then feeding such scaled data to an $\lseffnn$
which synthesizes a function in the LSE  (or, equivalently, 
$\lse_1$) class, having
 activations $s\mapsto\exp(s)$ in the hidden layer and 
$s\mapsto \log(s)$ in the output layer. 

Training and simulation of this type of $\lseffnn$ is implemented in a package we developed,
which works in conjunction with \texttt{Matlab}'s \texttt{Neural Network Toolbox}.

 In numerical practice, we shall fix $T>0$ and a value of $K$, train the network 
 with respect to the remaining model parameters as detailed above, 
and possibly iterate by adjusting $T$ and $K$, until a satisfactory fit is eventually found on validation data.
Here, the parameter $T$ controls the {\em smoothness} of the fitting function (as $T$ increases $f_T$ becomes ``smoother''), 
and the parameter $K$ controls the {\em complexity} of the model class (as $K$ increases $f_T$ becomes more complex). 

\section{Applications to physical examples\label{sec:appli}}
We next illustrate the proposed methodology with practical numerical examples.
In Section~\ref{sec:vibration} we find a convex model expressing the amount of 
vibration transmitted by a vehicle suspension system as a function of  its mechanical parameters. Similarly, in Section~\ref{sec:propane},
we derive a convex model 
 relating the peak power generated by the chemical reaction of propane combustion as a function of the  initial concentrations of all the involved chemical species. 

These models are first trained using the gathered data, and next used for  design (e.g., find concentrations that maximize power) by solving 
convex optimization and geometric programming problems via efficient numerical algorithms.
This two-step process (model training followed by model exploitation for design)
embodies an effective tool for performing data-driven optimization of complex physical processes. 

\subsection{Vibration transmitted by a vehicle suspension system\label{sec:vibration}}

In this numerical experiment, we considered the problem of identifying an $\lset$ and a 
$\gpos$ model for the amount of whole-body vibration transmitted by a vehicle suspension system having 11 degrees of freedom, as depicted in
Figure~\ref{fig:model}.
\begin{figure}[htb!]
\centering
\resizebox{0.35\textwidth}{!}{
\includegraphics{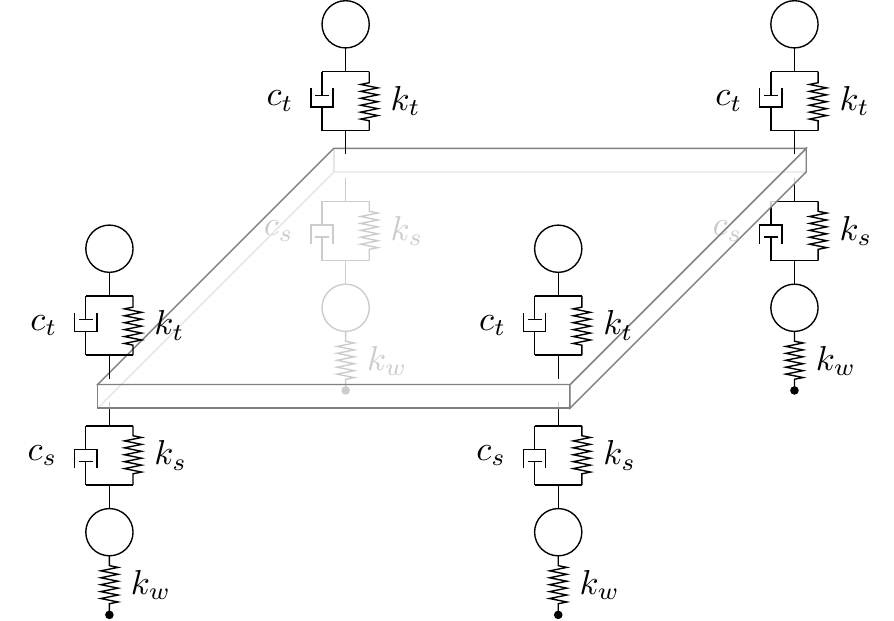}
}
\caption{Model for the vehicle  suspension system. 
\label{fig:model}}
\end{figure}

The model of the vehicle is taken from \cite{zareh2012semi} and includes the dynamics of the seats, the wheels, and
the suspension system. In order to measure the amount of vibration transmitted by the vehicle,
it is assumed that the left wheels of the vehicle are moving on a series of two bumps with constant speed, see \cite{zareh2012semi}
for further details on the vehicle model. 

The amount of whole-body vibration is measured following the international standard ISO-2631, \cite{ISO2631}.
Namely, the vertical acceleration of the left-set  $a(t)$
is  frequency-weighted through the function $H(t)$ following the directions given in Appendix~A of \cite{ISO2631}, 
thus obtaining the filtered signal 
\begin{equation*}
a_w(t)=\int_{0}^{t}H(t-\tau) \,a(\tau)\,\mathrm{d}\tau.
\end{equation*}
Then, the amount of transmitted vibration is computed as
\begin{equation*}
V= \left(\frac{1}{\Theta}\int_0^\Theta a_w^2(t)dt\right)^{\frac{1}{2}},
\end{equation*}
where $\Theta$ is the simulation time.
Clearly, $V$ is a complicated function of the input parameters $k_w$, $k_s$, $k_t$, $c_s$, and $c_t$, that can be evaluated by simulating the dynamics given in \cite{zareh2012semi}.
However, manipulation and parameter design
using direct and exact evaluation of $V$
 by integration of the dynamics  can be very costly 
 and time consuming. Therefore, 
 we are interested in obtaining a simpler model via an $\lseffnn$.

 It is here worth noticing that, in practice, we may  not know whether the function we are dealing with is convex, log-log-convex, or neither. Nevertheless, by fitting an $\lseffnn$ to the observed data we can obtain a convex (or log-log-convex) function approximation of  the data, and check directly a posteriori if the approximation error is satisfactory or not. 
 Further, certain types of physical data may suggest that a $\gpos$ model might be suitable for modeling them: it is the case with data where the inputs are physical quantities that are inherently positive (e.g., in this case, stiffnesses and damping coefficients) and likewise the output quantity is also inherently positive (e.g., the mean-squared acceleration level).
 
In this example, we identified both a model in $\lse_{T}$ and a model in $\mathrm{GPOS}_{T}$ for $V$.
Firstly, a set of 250 data points
\[\mS\doteq\{(\bbx_i, V_i) \}_{i=1}^{250},\]
has been gathered by simulating the dynamics of the system depicted in Figure~\ref{fig:model}
for randomly chosen values of $k_w$, $k_s$, $k_t$, $c_s$, and $c_t$ with the distributions shown in Table~\ref{tab:param}.

\begin{table}[htb!]
\caption{Distribution of the parameters in the multi-body simulations.\label{tab:param}}
\centering
{\renewcommand{\arraystretch}{1.1}
\renewcommand{\tabcolsep}{6pt}
\begin{tabular}{lcc}
\hline
Parameter & Distribution & Dimension\\
\hline
$k_w$ (stiffness of the wheel) & $\mN(175.41,\,17.1)$ & $\mathrm{kN/m}$\\
$k_s$ (stiffness of the suspension) & $\mN(17.424,\,1.72)$ & $\mathrm{kN/m}$\\
$k_t$ (stiffness of the set) & $\mN(1.747,\,0,17)$ & $\mathrm{kN/m}$\\
$c_s$ (damping of the suspension) & $\mN(1.465,\,0.15)$ & $\mathrm{kN\,s/m}$\\
$c_t$ (damping of the seat) & $\mN(0.697,\,0.07)$ & $\mathrm{kN\,s/m}$\\
\hline
\end{tabular}
}
\end{table}

\noindent
An  $\lseffnn$ with
 5 inputs 
 and $K=10$ hidden neurons
 has been implemented by interfacing
the \texttt{Neural Network} toolbox \cite{nntoolbox} with 
a  \texttt{Convex\_Neural\_Network} module that we developed, which provides
a \texttt{convexnet}  function that can be used for training the $\lseffnn$.
%
The temperature parameter $T$ 
has been determined via a campaign of several  cross validation experiments with varying values of $T$.
For the $\lse_{T}$, the best  model   in terms of mean absolute error has been obtained for $T=0.01$.
The same temperature value has been obtained  for the $\mathrm{GPOS}_{T}$ models.
%

After the training, the outputs of the $\lse_{T}$ and $\mathrm{GPOS}_{T}$ models
to the inputs $\{\bbx_i\}_{i=201}^{250}$ (which have not been used for training) have been
compared with $\{V_i \}_{i=201}^{250}$, with the outputs of a classical $\ffnn$ with 
symmetric sigmoid activation function for the hidden layer (with the same number of
hidden nodes) and linear activation function for the output layer and with the output of an $\ma$
function (with $10$ terms), which has been trained on the same data. In particular, the
$\ma$ function has been trained by using the heuristic given in \cite{magnani2009convex},
whereas the $\ffnn$, the $\lset$ and the $\gpos$ networks have been trained by using
the \texttt{Neural Network Toolbox}.
Figure~\ref{fig:num} and \ref{fig:errors} depict the estimates and the approximation errors obtained 
by using the $\ffnn$, $\ma$, $\lse_{T}$, $\mathrm{GPOS}_{T}$  models, whereas 
Table~\ref{tab:predErr} summarizes the error of each model.

\begin{figure}[htb!]
\centering
\includegraphics{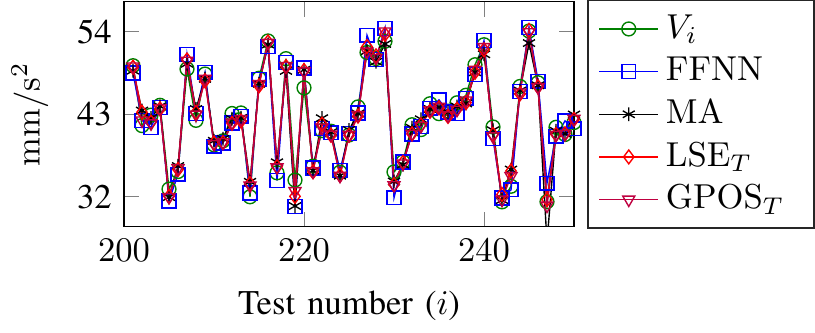}
\caption{Results of the numerical tests.\label{fig:num}}
\end{figure}

\begin{figure}[htb!]
\centering
\includegraphics{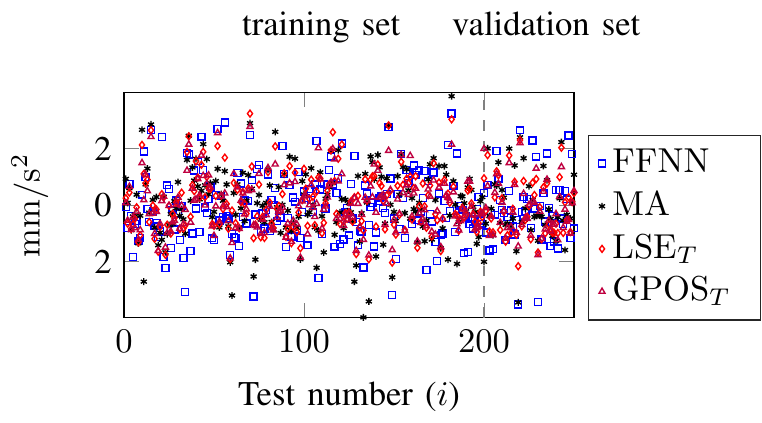}
\caption{Absolute approximation errors.\label{fig:errors}}
\end{figure}

\begin{table}[htb!]
\caption{Prediction errors\label{tab:predErr}}
\centering
{\renewcommand{\arraystretch}{1.2}
\renewcommand{\tabcolsep}{3pt}
\begin{tabular}{ccccc}
\hline
Method  & Mean abs. err. & Mean rel. err.& Max abs. err. & Max rel. err.\\
\hline
$\ffnn$  & $0.95\,\mathrm{mm/s^2}$& $2.36\%$ & $3.54\,\mathrm{mm/s^2}$ & $11.6\%$\\
$\ma$  & $0.96\,\mathrm{mm/s^2}$& $2.45\%$ & $5.38\,\mathrm{mm/s^2}$ & $20.73\%$\\
$\lse_{T}$  & $0.81\,\mathrm{mm/s^2}$& $1.98\%$& $3.25\,\mathrm{mm/s^2}$ & $7.85\%$\\
$\mathrm{GPOS}_{T}$ & $0.71\,\mathrm{mm/s^2}$& $1.69\%$& $2.79\,\mathrm{mm/s^2}$ & $5.87\%$\\
\hline
\end{tabular}
}
\end{table}

As shown by Table~\ref{tab:predErr}, the $\gpos$ model has the best performance in terms of absolute and relative errors. 

{The $\ffnn$ model $\phi$, the $\ma$ model $f_0$, 
the $\lse_{T}$ model $f_{T}$, and the $\mathrm{GPOS}_{T}$ model $\psi_{T}$ 
 have next been used to design the parameters $\bbx$ that minimize the amount of vibration~$V$.
Namely, letting $\bar{\bbx}$ be the mean value of the random variable used to find the models,
the nonlinear programming problem
\begin{equation} \label{eq:nonlinear}
\left\vert \begin{array}{rl}
\text{minimize }& \phi(\bbx)\\
\text{subject to } & 0.9 \,\bar{\bbx} \leqslant \bbx \leqslant 1.1\,\bar{\bbx} ,\\
\end{array}\right.
\end{equation}
has been solved by using the \texttt{Matlab} function \texttt{fmincon}. Similarly,
the convex optimization problems
\begin{equation} \label{eq:convexMA}
\left\vert \begin{array}{rl}
\text{minimize }& f_{0}(\bbx)\;\mbox{[or $f_T(\bbx)$]}\\
\text{subject to } &  0.9 \,\bar{\bbx} \leqslant \bbx\leqslant 1.1\,\bar{\bbx} ,\\
\end{array}\right.
\end{equation}
and the geometric program
\begin{equation} \label{eq:geometricProgram}
\left\vert \begin{array}{rlll}
\text{minimize }& \psi_{T}(\bbx)\\
\text{subject to } & 0.9 \,\bar{x}_i\, x_i^{-1}\leqslant 1, &\quad i=1,\dots,5,\\
& 1.1 \,\bar{x}_i^{-1}\, x_i\leqslant 1, &\quad i=1,\dots,5,\\
\end{array}\right.
\end{equation}
have been solved by using \texttt{CVX}, a package for solving convex and geometric programs \cite{boyd2007tutorial,cvx,gb08,duffin1967geometric}.
Then, the dynamics of the multibody system depicted in Figure~\ref{fig:model} have been simulated 
with the optimal values gathered by solving either~\eqref{eq:nonlinear}, \eqref{eq:convexMA}, 
or~\eqref{eq:geometricProgram}. 
The computation of the solution to~\eqref{eq:nonlinear} required $4.813\,\mathrm{s}$ (that is larger than
the computation times reported in Table~\ref{tab:optimVal} due to the fact that convex optimization tools cannot be used)
and the computed optimal solution led to an amount of total vibration equal to $46.525\,\mathrm{mm/s^2}$.
On the other hand, the results obtained by solving   \eqref{eq:convexMA}, 
and~\eqref{eq:geometricProgram}
are reported in Table~\ref{tab:optimVal}. 
\begin{table}[htb!]
\caption{Results of the simulations with the optimal values\label{tab:optimVal}}
\centering
{\renewcommand{\arraystretch}{1.2}
\renewcommand{\tabcolsep}{8pt}
\begin{tabular}{ccc}
\hline
Problem solved  & Computing time & Amount of vibration\\
\hline
$f_0$
&   $0.922\,\mathrm{s}$ &  $30.513\,\mathrm{mm/s^2}$\\
$f_T$ & $1.628\,\mathrm{s}$ & $29.723\,\mathrm{mm/s^2}$ \\
$\psi_T$& $0.554\,\mathrm{s}$ & $29.609\,\mathrm{mm/s^2}$\\
\hline
\end{tabular}}
\end{table}

The results given in Table~\ref{tab:optimVal} highlight the effectiveness of the $\gpos$ model, which indeed yields the best design in the least computing time.}


\subsection{Combustion of propane\label{sec:propane}}

In this numerical experiment, we considered the problem of identifying a convex 
and a log-log-convex model for the peak power 
generated through the combustion of propane. We considered the reaction network for the combustion
of propane presented in \cite{jachimowski1984chemical}, which consists of 83
reactions and 29 chemical species, see Figure~\ref{fig:stoic} for a graphical representation of the stoichiometric matrix
of this chemical reaction network.

\begin{figure}[thb]
\centering
\resizebox{0.22\textwidth}{!}{
\includegraphics{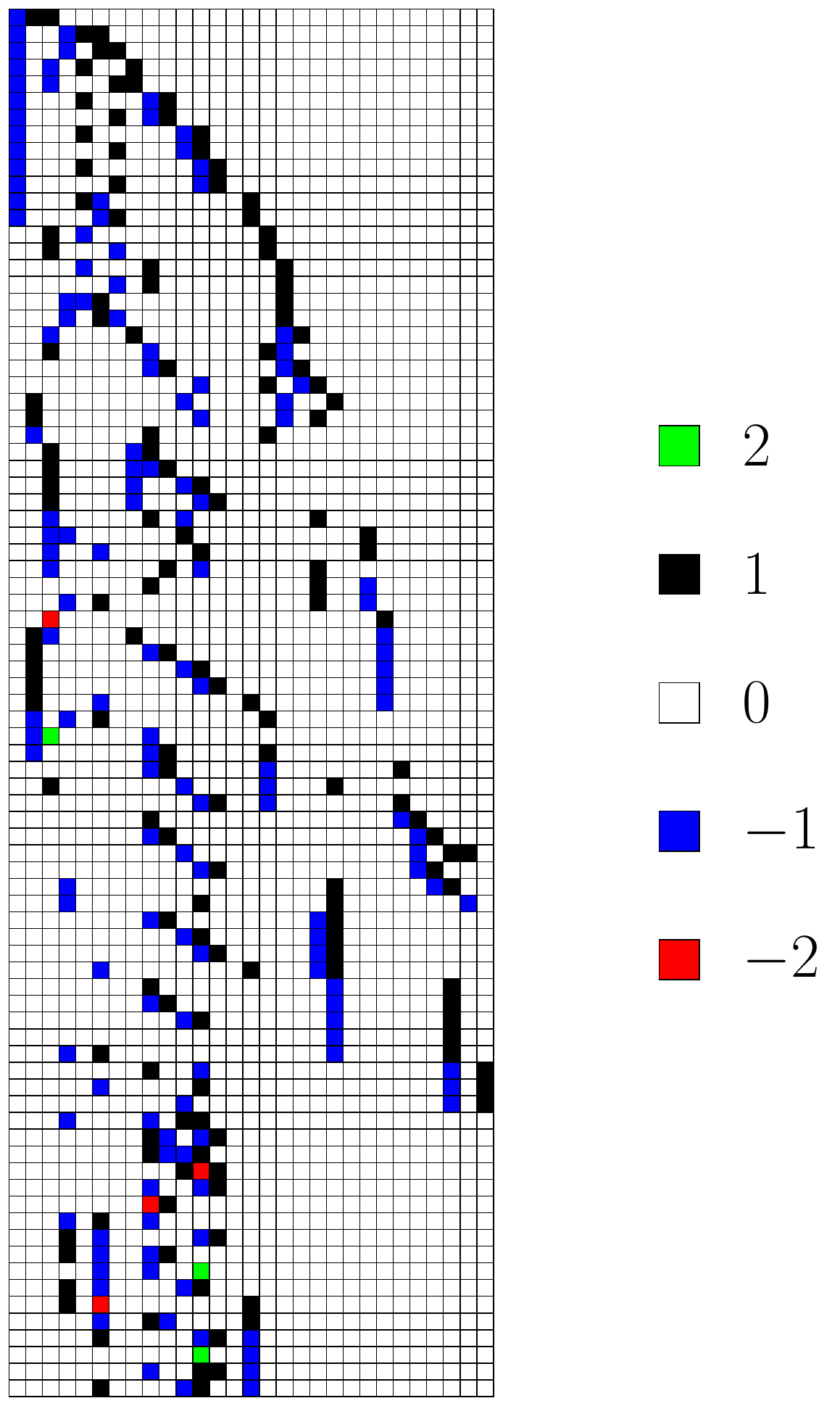}
}
\caption{Stoichiometric matrix of the reaction network for the combustion of propane, in which
each row corresponds to a different chemical reaction and each column corresponds to 
a different chemical species.\label{fig:stoic}}
\end{figure}

The instantaneous power generated by the combustion  is approximatively given by
\begin{subequations}\label{eq:totalpower}
\begin{equation}\label{eq:power}
P(t)=\Delta_c H^e \frac{\mathrm{d}\, m(t)}{\mathrm{d}\,t},
\end{equation}
where $\Delta_c H^e$ is the calorific value of the propane ($\simeq 2220\,\mathrm{kJ}/\mathrm{mol}$)
and $m(t)$ denotes the number of moles of propane. Hence, the peak power generated by this reaction~is
\begin{equation}\label{eq:peakpower}
\overline{P}\doteq\max_{t\in\Rp} P(t).
\end{equation}
\end{subequations}
Clearly, $\overline{P}$ is a function of the initial concentrations $\bbx$ of all the involved chemical species
and can be obtained by performing exact numerical simulations of the chemical network (e.g., by using
the stochastic simulation algorithm given in \cite{possieri2018stochastic}), taking averages to determine
the mean behavior of $m(t)$, and using \eqref{eq:totalpower}. However, performing all these computations
is rather costly, due to the fact that a large number of simulations has to be performed in order to
take average. Hence, in order to maximize the effectiveness of the combustion, it is more convenient
to obtain a simplified ``surrogate'' model relating $\bbx$ and $\overline{P}$. In particular, convex and log-log-convex models relating $\bbx$ 
with $\overline{P}^{-1}$ seem to be  appealing since they allow the design of the initial concentrations
that maximizes $\overline{P}$ by means of computationally efficient algorithms.

In this example, we identify a model in $\lse_{T}$  and a model
in $\mathrm{GPOS}_{T}$ for  ${\overline{P}}^{-1}$ as a function of $\bbx$. 
We observe that, also in this example, $\mathrm{GPOS}_{T}$ models appear to be potentially well adapted to the physics of the problem, since all input variables are positive concentrations of chemicals, and the output (peak power)  is also positive.
First, a collection of 500 data points \[\mS\doteq\{(\bbx_i,\overline{P}_i^{-1}) \}_{i=1}^{500}\] has been gathered
by choosing randomly the initial condition $\bbx_i$ of the 
chemical reaction network with uniform distribution in the interval $[1.494,1.827]\,\mathrm{pmol/m^3}$,
by performing $1000$ simulations of the chemical reaction network through the algorithm given
in \cite{possieri2018stochastic}, by taking averages to determine the expected time behavior of
$m(t)$, and using \eqref{eq:totalpower} to determine the value of $\overline{P}_i$, $i=1,\dots,500$.

Then, the function \texttt{convexnet} of the toolbox  \texttt{Convex\_Neural\_Network} has been used 
to design an $\lseffnn$  with
$n= 29$ input nodes, 1 output node, and 1 hidden layer with $K=3$  neurons.
Several cross-validation experiments have been executed preliminarily in order to determine a satisfactory value for the temperature parameter $T$, which resulted to be $T=0.01$ for $\lse_{T}$ models and 
$T=0.005$ for $\mathrm{GPOS}_{T}$ models.
%

After training the network,  we considered  the inputs $\{\bbx_i\}_{i=251}^{500}$ (which have not been used for training)
and computed the corresponding  outputs for  the $\lse_{T}$ model $f_T$ and of the $\mathrm{GPOS}_{T}$ model $\psi_T$ . These outputs 
 are compared with $\{\overline{P}_i^{-1} \}_{i=251}^{500}$, with the outputs of a classical $\ffnn$ with 
symmetric sigmoid activation function for the hidden layer (with the same number of
hidden nodes) and linear activation function for the output layer, which has been trained by using the 
same data, and with the outputs of an $\ma$ function $f_0$  (with 3 terms)
that has been trained on the same data by using the 
heuristic given in \cite{magnani2009convex}.
Figure~\ref{fig:num2} and \ref{fig:errors2} depict the estimates and the approximation errors obtained 
by using the $\ffnn$, $\ma$, $\lse_{T}$ and $\mathrm{GPOS}_{T}$  models, whereas 
Table~\ref{tab:predErr2}
summarizes the prediction error of each model.

\begin{figure}[htb!]
\centering
\includegraphics{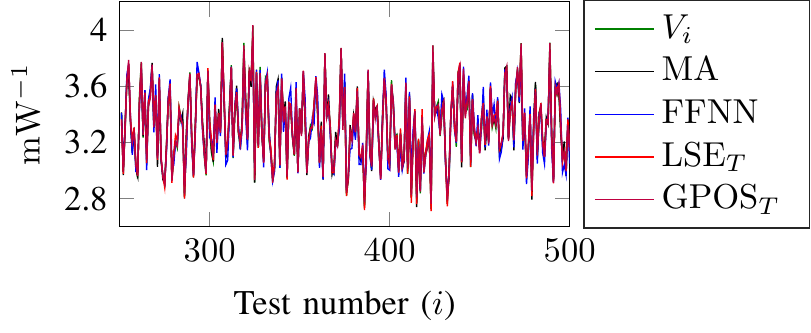}
\caption{Results of the numerical tests.\label{fig:num2}}
\end{figure}

\begin{figure}[htb!]
\centering
\includegraphics{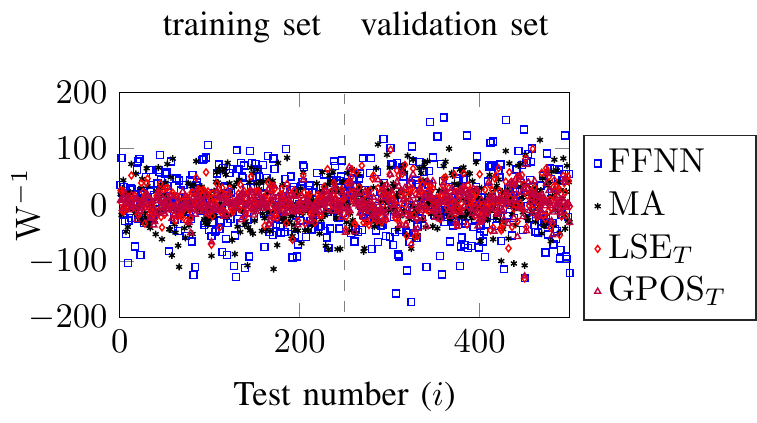}
\caption{Approximation errors.\label{fig:errors2}}
\end{figure}

\begin{table}[htb!]
\caption{Prediction errors\label{tab:predErr2}}
\centering
{\renewcommand{\arraystretch}{1.2}
\renewcommand{\tabcolsep}{4pt}
\begin{tabular}{ccccc}
\hline
Method & Mean abs. err. & Mean rel. err. & Max abs. err. & Max rel. err.\\
\hline
$\ffnn$ & $42.15\,\mathrm{W^{-1}}$ & $1.28\%$ & $172.75\,\mathrm{W^{-1}}$& $4.62\%$ \\
$\ma$ & $30.38\,\mathrm{W^{-1}}$ & $9.27\%$ & $115.81\,\mathrm{W^{-1}}$& $4.33\%$ \\
$\lse_{T}$  & $20.48\,\mathrm{W^{-1}}$ & $0.62\%$ & $130.63\,\mathrm{W^{-1}}$& $4.18\%$ \\
$\mathrm{GPOS}_{T}$  & $14.94\,\mathrm{W^{-1}}$ & $0.45\%$ & $126.31\,\mathrm{W^{-1}}$& $4.03\%$ \\
\hline
\end{tabular}
}
\end{table}

As shown by Table~\ref{tab:predErr2}, the $\lse_{T}$ and the $\mathrm{GPOS}_{T}$ models
have improved prediction capabilities with respect to the classical $\ffnn$ 
and the $\ma$ model. In particular, the
model in $\mathrm{GPOS}_{T}$ presents the best approximation performance.
Moreover, the model $f_0$ in $\ma$ , the model $f_{T}$ in $\lse_{T}$ and the model $\psi_{T}$ in $\mathrm{GPOS}_{T}$,
can be used
to efficiently design the initial concentrations $\bbx$ that are within the considered range and that maximize the peak power.
In fact, the convex optimization problems
\begin{equation} \label{eq:convexMA2}
\left\vert \begin{array}{rl}
\text{minimize }& f_{0}(\bbx) \;\mbox{[or $f_{T}(\bbx)$]}\\
\text{subject to } &  1.49\cdot 10^{-12} \leqslant \bbx \leqslant 1.83\cdot 10^{-12} , 
\end{array}\right.
\end{equation}
as well as the geometric program
\begin{equation} \label{eq:geometricProgram2}
\left\vert \begin{array}{rll}
\text{minimize }& \psi_{T}(\bbx)\\
\text{subject to } & 1.49\cdot 10^{-12}\, x_i^{-1}\geqslant1,& i=1,\dots,29,\\
& 1.83\cdot 10^{-12} \, x_i\geqslant1,& i=1,\dots,29,\\
\end{array}\right.
\end{equation}
can be efficiently solved by using any solver able to deal with convex optimization problems.
On the other hand, letting $\phi$ be the $\ffnn$ model, solving a problem
of the form \eqref{eq:convexMA2} with $\phi$ as the objective
may be rather challenging due to the fact that it need not be (and generically is not) convex.
In fact, by attempting at solving such nonlinear programming problem via 
\texttt{fmincon} we failed,  whereas we have
been able to find the solutions of problems~\eqref{eq:convexMA2} and \eqref{eq:geometricProgram2}
by using the \texttt{Matlab} toolbox \texttt{CVX}.
Table~\ref{tab:optimVal2} reports the computing time required to determine such solutions
and the corresponding peak power obtained by simulating the chemical reaction network.

\begin{table}[htb!]
\caption{Results of the simulations with the optimal values\label{tab:optimVal2}}
\centering
{\renewcommand{\arraystretch}{1.2}
\renewcommand{\tabcolsep}{8pt}
\begin{tabular}{ccc}
\hline
Problem solved  & Computing time & Peak power\\
\hline
$f_0$ & $1.279\mathrm{s}$ & $0.363\,\mathrm{m W}$ \\
$f_t$ & $2.968\mathrm{s}$ & $0.368\,\mathrm{m W}$ \\
$\psi_T$ & $0.655\,\mathrm{s}$ & $0.369\,\mathrm{mW}$\\
\hline
\end{tabular}}
\end{table}

As shown by Table~\ref{tab:optimVal2}, the model in $\mathrm{GPOS}_{T}$
presents the best performance in the considered example. 


\section{Conclusions\label{sec:concl}}
A feedforward neural network with exponential activation functions in the inner layer and logarithmic activation in the output layer
can approximate with arbitrary precision any convex function on a convex and compact input domain. 
Similarly,  any log-log-convex function can be approximated to arbitrary relative precision by a class of generalized posynomial functions. 
This allows us to construct convex (or log-log-convex) models that approximate  observed data, with the advantage over standard feedforward networks that the synthesised input-output map is convex (or log-log-convex) in the input variables, which makes it readily amenable to efficient optimization via convex or geometric programming.

The techniques given in this paper enable data-driven
optimization-based design methods that apply convex optimization on a surrogate
model obtained from data.
Of course, some data might be more suitable than other to be approximated via
convex or log-log-convex models: if the (unknown) data generating function underlying the observed data is indeed convex (or log-log-convex), then we may expect very good results in the fitting via LSE$_T$  functions (or $\gpos$ functions). Actually, even when 
convexity or log-log-convexity of the data generating function is not known a priori, we can find in many cases a data fit that is of quality comparable, or even better, than the one obtained via general non-convex neural network models, with the clear advantage of having a model possessing the additional and desirable feature of convexity.  


\bibliographystyle{ieeetr}

\begin{IEEEbiographynophoto}{Giuseppe C. Calafiore}
 (S '14, F '18) received the ``Laurea'' degree in Electrical Engineering from Politecnico di Torino in 1993, and the Ph.D. degree in Information and System Theory from Politecnico di Torino, in 1997. He is with the faculty of Dipartimento di Electronics and Telecommunications, Politecnico di Torino, where he currently serves as a full professor and coordinator of the Systems and Data Science lab.
Dr. Calafiore held several visiting positions at international institutions: at the Information Systems Laboratory (ISL), Stanford University, California, in 1995; at the Ecole Nationale Sup\'erieure de Techniques Avance\'es (ENSTA), Paris, in 1998; and at the University of California at Berkeley, in 1999, 2003 and 2007. He had an appointment as a Senior Fellow at the Institute of Pure and Applied Mathematics (IPAM), University of California at Los Angeles, in 2010. He had appointments as a Visiting Professor at EECS UC Berkeley in 2017 and at the Haas Business School in 2018.
He is a Fellow of the Italian National Research Council (CNR). He has been an Associate Editor for the IEEE Transactions on Systems, Man, and Cybernetics (T-SMC), for the IEEE Transactions on Automation Science and Engineering (T-ASE), and for Journal Europ\'een des Syst\'emes Automatis\'es (JESA). He currently serves as an Associate Editor for the IEEE Transactions on Automatic Control.
 Dr. Calafiore is the author of more than 180 journal and conference proceedings papers, and of eight books. He is a fellow member of the IEEE since 2018. He received the IEEE Control System Society ``George S. Axelby'' Outstanding Paper Award in 2008.  His research interests are in the fields of convex optimization, randomized algorithms, machine learning, computational finance, and identification and control of uncertain systems.

\end{IEEEbiographynophoto}

\begin{IEEEbiographynophoto}{Stephane Gaubert}
graduated from \'Ecole
Polytechnique, Palaiseau, in 1988. He got a PhD degree in Mathematics
and Automatic Control from \'Ecole Nationale Supérieure des Mines de
Paris in 1992. He is senior research scientist (Directeur de
Recherche) at INRIA Saclay -- \^Ile-de-France and member of CMAP (Centre
de Math\'ematiques Appliqu\'ees, \'Ecole Polytechnique, CNRS), head of a joint
research team, teaching at \'Ecole Polytechnique.  He coordinates
the Gaspard Monge corporate sponsorship Program for Optimization and Operations Research (PGMO),
of Fondation Math\'ematique Hadamard. 
His interests include tropical geometry, optimization,  game theory, monotone or nonexpansive
dynamical systems, and applications of mathematics to decision making and to the verification of programs
or systems. 
\end{IEEEbiographynophoto}

\begin{IEEEbiographynophoto}{Corrado Possieri}
received his bachelor's and master's degrees in Medical engineering and his Ph.D. degree in Computer Science, Control and Geoinformation from the University of Roma Tor Vergata, Italy, in 2011, 2013, and 2016, respectively. 
From September 2015 to June 2016, he visited the University of California, Santa Barbara (UCSB). 
Currently, he is Assistant Professor at the Politecnico di Torino.
He is a member of the IFAC TC on Control Design.
His research interests include stability and control of hybrid systems, the application of computational algebraic geometry techniques to control problems, stochastic systems, and optimization.
\end{IEEEbiographynophoto}

\end{document}